\ifdefined\pdfdraftmode
\fi
\ifdefined\pdfvariable
  \pdfvariable draftmode 0
\fi
\documentclass[11pt]{article}

\PassOptionsToPackage{final}{graphics}
\PassOptionsToPackage{final}{graphicx}

\usepackage[utf8]{inputenc}
\usepackage[T1]{fontenc}
\usepackage[margin=1in]{geometry}
\usepackage{times}
\usepackage[numbers,sort&compress]{natbib}
\usepackage{amsthm}
\usepackage{url}
\usepackage{booktabs}
\usepackage{amsmath}
\usepackage{mathtools}
\usepackage{amssymb}
\usepackage{amsfonts}
\usepackage{graphicx}
\graphicspath{{figures/}}
\usepackage{xcolor}
\usepackage{enumitem}
\usepackage{placeins}
\usepackage{flafter}
\usepackage{float}
\usepackage{wrapfig}
\usepackage{algorithm2e}
\newfloat{routedalgorithm}{tbp}{loa}
\floatname{routedalgorithm}{Algorithm}
\usepackage{hyperref}
\hypersetup{hidelinks}

\makeatletter
\AtBeginDocument{%
    \setkeys{Gin}{draft=false}%
    \Gin@draftfalse
}
\makeatother

\title{The Dynamics of Collapse and Recovery in Clustered Federated Learning}
\author{Priyanka Sinha \qquad Snigdhansu Chatterjee\\[0.4em]
\small University of Maryland, Baltimore County\\
\small \texttt{psinha1@umbc.edu} \qquad \texttt{snigchat@umbc.edu}}
\date{}

\newtheorem{theorem}{Theorem}

\theoremstyle{definition}

\theoremstyle{remark}

\begin{document}

\maketitle

\begin{abstract}
Clustered federated learning (CFL) can suffer from cluster collapse. In particular, under unfavorable initialization, early client assignments concentrate heterogeneous groups on one cluster model, which is then trained toward a pooled compromise, while another model receives no updates. Since the current models determine client routing and routing determines subsequent model updates, we formulate clustered federated learning as control of an unknown dynamical system, with cluster models as the evolving state, client routing patterns as the control input, and routing interventions as state dependent controllers. In a two group linear model, we identify an initialization region in which greedy routing makes collapse self reinforcing. We then derive finite horizon error dynamics showing how routing interventions can reverse collapse by repeatedly training the starved model, primarily with updates from its compatible group. Motivated by this condition, each client uses blocked EXP3 to choose between greedy routing and interventions that redirect updates toward unused models. Experiments show that with simpler predictors and faster specialist formation, autonomous controller selection recovers usable specialists from unfavorable starts that trap standard CFL baselines. In transformed FEMNIST, the synchronized-entry configuration with persistent commitment recovers usable specialists in settings where the autonomous controller configuration often does not. Code and scripts for reproducing the experiments and figures are available at \url{https://github.com/prsinha1/FedRecover}.

\end{abstract}

\section{Introduction}
\label{sec:introduction}

Clustered federated learning (CFL) addresses client heterogeneity by maintaining several shared models and assigning each client to the model that best fits its local data \cite{ghosh2021efficientframeworkclusteredfederated,mansour2020approachespersonalizationapplicationsfederated}. However, under an unfavorable initialization, CFL is widely known to be prone to cluster collapse, where clients from several latent groups may initially prefer the same model, while other models are starved of training. Under strong heterogeneity, such a surviving model is trained across conflicting client distributions and becomes a compromised pooled solution, leading to optimization instability \cite{NEURIPS2023_8668fdc7}. Initialization sensitivity is therefore a consequence of a deeper dynamical failure, where early routing decisions determine which models improve, and those improvements determine later routing decisions.


Existing CFL methods either assume or seek favorable initialization, or modify the clustering architecture to mitigate collapse \cite{ghosh2021efficientframeworkclusteredfederated, nguyen2022beginimpactpretraininginitialization, harshvardhan2022improvedalgorithmclusteredfederated, NEURIPS2023_8668fdc7}. We instead ask whether a fixed collection of cluster models can recover after unfavorable initialization has already induced collapsed routing.

To answer this question, we make the feedback loop between client side routing and server side model updates explicit by viewing CFL as a controlled dynamical system. The cluster models \(W^t\) form the state, the routing pattern \(M^t\) is the control input, and the server update defines the state evolution, $ W^{t+1}=\mathcal F(W^t,M^t;\mathcal E), $ where \(\mathcal E\) represents the unknown client objectives and latent group structure. We introduce \(\Phi^t\) as a feedback controller that maps observable client information to routing decisions inducing \(M^t\). Preventing collapse then becomes a control problem over an unknown and partially observed dynamical system, paralleling the control theoretic treatment of interactive personalized recommendation systems \cite{depasquale2026recommendersystemscontrolsystems}.

\noindent\textbf{Our contributions.}
First, in a two group linear model, we characterize an unfavorable initialization basin in which greedy routing permanently starves one cluster model while the active model converges to a pooled solution. Second, for arbitrary realized routing sequences, we derive a finite horizon recovery condition requiring sufficient updates to the disadvantaged model while controlling incompatible client contributions, which we call persistent selective exposure. Third, motivated by this condition, we design complementary routing controllers and let each client learn among them online. Experiments show that decentralized control recovers distinct usable models when specialization forms quickly. For transformed FEMNIST, we compare autonomous controller selection with synchronized entry under the specified controller settings.


\section{Initialization induced failure and recovery}
\label{sec:routing_recovery}

We consider $n$ clients each coming from one of $G$ heterogeneous latent groups $I_1,\ldots,I_G$. A client $i\in I_g$ has local distribution $\mathcal D_i$ and shares the group population objective $L_g(w)=\mathbb E_{z\sim\mathcal D_i}[\ell(w;z)]$. Group membership is hidden from both clients and server. At round $t$, the server maintains cluster models $W^t=(w_1^t,\ldots,w_K^t)$. Greedy CFL assigns the client to a lowest loss model and updates each cluster using only the clients assigned to it.

\textbf{Favorable and unfavorable initialization basins:}
To demonstrate how the assignment and update feedback determines the training trajectory, consider two latent groups, relabeled $0$ and $1$, and two cluster models. For group $g\in\{0,1\}$, samples $(x,y)$ satisfy $y=x^\top\theta_g+\epsilon$, where $\theta_g\in\mathbb R^d$ is the group specific regression parameter. We assume $\mathbb E[x]=0$, $\mathbb E[xx^\top]=I$, and independent zero mean observation noise. Under squared loss, the population objective is, up to an additive noise constant, $L_g(w)=\tfrac12\|w-\theta_g\|_2^2$, with $\nabla L_g(w)=w-\theta_g$. Thus $\theta_g$ is the unique population minimizer for group $g$. Let $\Delta=\|\theta_1-\theta_0\|_2$, and let the group proportions be $\alpha$ and $1-\alpha$.

\textbf{Initial preference geometry and success criteria:}
An initialization has \emph{common preference} when both groups initially prefer the same cluster. Formally, define
\begin{equation}
\mathcal C_{\rm common}
:=
\bigcup_{k\in\{0,1\}}
\left\{
W^0:
L_g(w_k^0)<L_g(w_{1-k}^0),\ \forall g\in\{0,1\}
\right\}.
\label{eq:common-preference}
\end{equation}
Otherwise, apart from ties, the two groups initially prefer different clusters. These labels describe only the first greedy assignment and do not determine whether later updates preserve it. Define the permutation invariant recovery error
\begin{equation}
R_T
:=
\min_{m\in\mathfrak S_2}
\max_{g\in\{0,1\}}
\left\|w_{m(g)}^T-\theta_g\right\|_2.
\label{eq:recovery-error}
\end{equation}
The system achieves $r$ recovery at round $T$ when $R_T\le r$. When $r<\Delta/2$, each group strictly prefers its matched model under terminal greedy routing. Thus recovery gives each group a distinct usable specialist rather than a single pooled compromise. In the balanced experimental setting, each matched model is also better for its group than the pooled parameter.

Define the certified specialization basin
\begin{equation}
\mathcal B_{\rm specialize}
:=
\bigcup_{m\in\mathfrak S_2}
\left\{
W^0=(w_0^0,w_1^0)\in\mathbb R^d\times\mathbb R^d:
\sum_{g=0}^{1}\left\|w_{m(g)}^0-\theta_g\right\|_2<\Delta
\right\}.
\label{eq:specialization-basin}
\end{equation}

\begin{theorem}[Greedy specialization basin]
\label{thm:specialist_basin}
Suppose $W^0\in\mathcal B_{\rm specialize}$ and $0<\eta\le1$. Let $m\in\mathfrak S_2$ witness membership in \eqref{eq:specialization-basin}. Then group $g$ greedily selects cluster $m(g)$ at every round, and
\[
w_{m(g)}^{t+1}=w_{m(g)}^t-\eta(w_{m(g)}^t-\theta_g),
\qquad
w_{m(g)}^t\longrightarrow\theta_g,
\qquad g\in\{0,1\}.
\]
\end{theorem}

\textit{Proof:} See Appendix~\ref{app:proof_greedy_specialization}.

The basin condition certifies a favorable initialization in which the groups initially prefer different models and the induced updates preserve that separation. We next identify an initialization region with the opposite behavior. Let $\bar\theta_\alpha=\alpha\theta_0+(1-\alpha)\theta_1$. For $j\in\{0,1\}$, define the basin in which cluster $j$ is starved as
\begin{equation}
\begin{aligned}
\mathcal B_{{\rm starve},j}
&:=
\left\{
W^0=(w_0^0,w_1^0)\in\mathbb R^d\times\mathbb R^d:
L_g(w_j^0)>
\max\!\left\{
L_g(w_{1-j}^0),L_g(\bar\theta_\alpha)
\right\},
\ \forall g\in\{0,1\}
\right\},
\\
\mathcal B_{\rm starve}
&:=
\mathcal B_{{\rm starve},0}\cup\mathcal B_{{\rm starve},1}.
\end{aligned}
\label{eq:starvation-basin}
\end{equation}
Within either component, both groups select the same active cluster, the resulting updates preserve that common preference, and the other cluster remains permanently starved.

\begin{theorem}[Greedy starvation basin]
\label{thm:starvation_basin}
Let $0<\eta\le1$. If $W^0\in\mathcal B_{{\rm starve},j}$, then both groups route to cluster $1-j$ at every round, $w_j^t=w_j^0$, and $w_{1-j}^{t+1}=w_{1-j}^t-\eta(w_{1-j}^t-\bar\theta_\alpha)$. Consequently, $w_{1-j}^t\to\bar\theta_\alpha$ while cluster $j$ remains permanently starved.
\end{theorem}

\textit{Proof:} See Appendix~\ref{app:proof_greedy_starvation}.

\begin{figure}[htbp]
    \centering
    \includegraphics[width=0.78\linewidth]{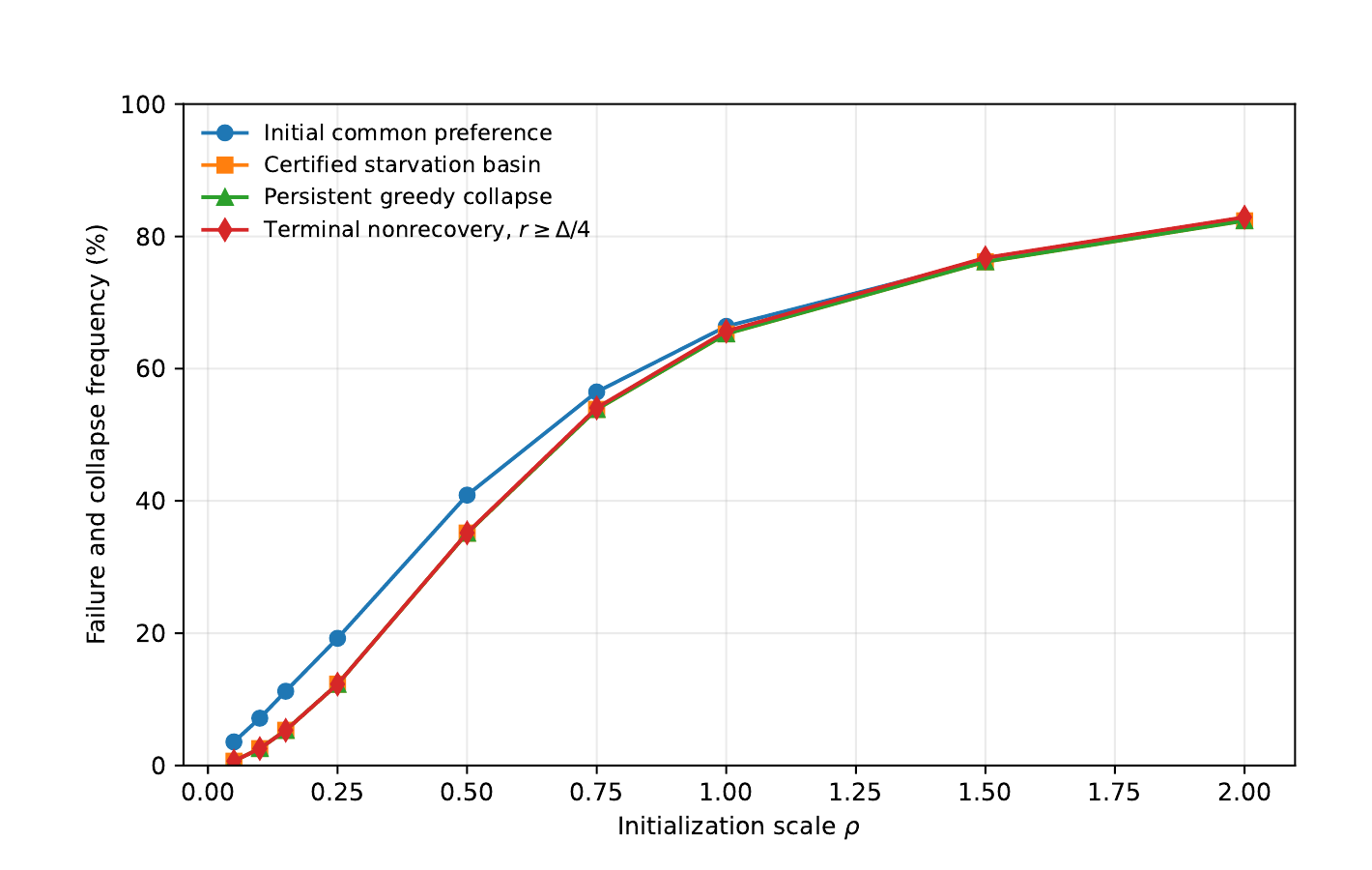}
    \caption{Empirical frequency of initial common preference, membership in the certified starvation basin $\mathcal B_{\rm starve}$, persistent greedy collapse, and terminal nonrecovery as functions of the initialization spread $\rho$.}
    \label{fig:starvation-frequency}
\end{figure}

To compare the theorem with the realized dynamics, we sample $20{,}000$ initialization pairs per value of $\rho$ from $w_k^0=\bar\theta+\rho z_k$, where $z_k\sim\mathcal N(0,I)$ and $\bar\theta=(\theta_0+\theta_1)/2$. In Figure~\ref{fig:starvation-frequency}, initial common preference is the event $W^0\in\mathcal C_{\rm common}$ from \eqref{eq:common-preference}. It records only whether the first assignments coincide. The certified starvation curve is the frequency of $W^0\in\mathcal B_{\rm starve}$ and therefore counts initializations for which Theorem~\ref{thm:starvation_basin} guarantees permanent starvation. Persistent greedy collapse means that greedy routing uses only one cluster throughout the $200$ simulated rounds. Terminal nonrecovery is the event $R_{200}\ge\Delta/4$.

All four frequencies increase with $\rho$ and exceed $80$ percent at the largest tested spread. Common preference is more frequent than the other events at smaller $\rho$, showing that a shared first choice alone does not certify failure. Every sampled initialization in $\mathcal B_{\rm starve}$ remains persistently collapsed, while some common preference initializations outside the certified basin escape as training changes the loss ordering. Persistent collapse and terminal nonrecovery closely follow the certified basin frequency. Thus Figure~\ref{fig:starvation-frequency} shows both that the sufficient basin has substantial probability under pooled centered Gaussian initialization and that its deterministic conclusion accurately predicts the simulated greedy trajectories.

\textbf{Routing controlled dynamics:}
The starvation theorem describes the trajectory induced by greedy routing. To characterize what an intervention must change, let $a_i^t\in\{1,\ldots,K\}$ be client $i$'s sampled cluster and let $q_i^t=e_{a_i^t}$ be its one hot executed route. Write $N_g=\sum_{i\in I_g}n_i$, $N=\sum_i n_i$, and $\alpha_g=N_g/N$. Define
\begin{equation}
M_{gk}^t
=\frac1{N_g}\sum_{i\in I_g}n_iq_{ik}^t
=\frac1{N_g}\sum_{i\in I_g}n_i\mathbf1\{a_i^t=k\},
\label{eq:routing-mass}
\end{equation}
\begin{equation}
\begin{aligned}
s_k^t
&=\mathbf1\!\left\{\sum_h\alpha_hM_{hk}^t>0\right\}
=\mathbf1\!\left\{\sum_i n_i\mathbf1\{a_i^t=k\}>0\right\},\\
\beta_{gk}^t
&=\frac{\alpha_gM_{gk}^t}{\sum_h\alpha_hM_{hk}^t}
=\frac{\sum_{i\in I_g}n_i\mathbf1\{a_i^t=k\}}
{\sum_i n_i\mathbf1\{a_i^t=k\}}.
\end{aligned}
\label{eq:routing-summary}
\end{equation}
The second equality in each line makes the hard routing and sample count weighting explicit. Here $M_{gk}^t$ is the fraction of group $g$'s sample mass routed to cluster $k$, $s_k^t$ indicates whether cluster $k$ receives any client update, and $\beta_{gk}^t$ is the sample weighted contribution of group $g$ to that cluster's update when $s_k^t=1$. The routing induced server update remains
\begin{equation}
w_k^{t+1}
=
\begin{cases}
w_k^t-\eta\sum_g\beta_{gk}^t\nabla L_g(w_k^t), & s_k^t=1,\\
w_k^t, & s_k^t=0.
\end{cases}
\label{eq:routing_controlled_dynamics}
\end{equation}
If no client samples cluster $k$, then $M_{gk}^t=0$ for every group, $s_k^t=0$, and the cluster remains frozen. A routing intervention can alter this trajectory by giving an otherwise unused cluster positive sampling probability and thereby producing executed updates for it.

\textbf{Matched error dynamics:}
Fix a matching $m$ between groups and clusters, and define the matched error $e_g^t=w_{m(g)}^t-\theta_g$, activity $a_g^t=s_{m(g)}^t$, and incompatible update contribution $c_g^t=\sum_{h\ne g}\beta_{h,m(g)}^t(\theta_h-\theta_g)$. Substituting these in Equation~\eqref{eq:routing_controlled_dynamics} gives
\begin{equation}
e_g^{t+1}
=
\begin{cases}
(1-\eta)e_g^t+\eta c_g^t, & a_g^t=1,\\
e_g^t, & a_g^t=0.
\end{cases}
\label{eq:matched_error_dynamics}
\end{equation}
\textbf{Persistent selective exposure:} When the matched cluster is inactive, its error cannot contract. When it is active, its error contracts toward the group specialist, while incompatible updates displace it through $c_g^t$. Recovery therefore requires repeated training of the matched cluster while its update composition remains dominated by the group it is intended to serve. We call this combination \emph{persistent selective exposure}. For $E_g(t,T):=\sum_{\tau=t}^{T-1}a_g^\tau$, iterating Equation~\eqref{eq:matched_error_dynamics} gives the following condition.

\begin{theorem}[Finite horizon recovery under selective exposure]
\label{thm:selective_exposure}
Fix a matching $m$, let $0<\eta\le1$, and consider any realized routing sequence through round $T$. If, for each $g\in\{0,1\}$,
\begin{equation}
(1-\eta)^{E_g(0,T)}\|e_g^0\|_2
+\eta\sum_{t=0}^{T-1}a_g^t(1-\eta)^{E_g(t+1,T)}\|c_g^t\|_2\le r,
\label{eq:selective_exposure_condition}
\end{equation}
then $\|w_{m(g)}^T-\theta_g\|_2\le r$. If $r<\Delta/2$, terminal greedy routing strictly separates the two groups. The proof and excess risk conclusion appear in Appendix~\ref{app:proof_finite_horizon_recovery}.
\end{theorem}

\section{Adaptive routing control}
\label{sec:routing_control}

Theorem~\ref{thm:selective_exposure} identifies three properties that routing must produce. \emph{Exposure} activates a matched cluster, making $a_g^t=1$ so its error can contract. \emph{Persistence} accumulates enough active rounds in $E_g(0,T)$ to reduce the initial error term in \eqref{eq:selective_exposure_condition}. \emph{Selectivity} keeps $\|c_g^t\|_2$ small by directing the cluster's updates primarily from its compatible group. At the controller level, exploration creates exposure, exploitation of a useful lowest loss model promotes selectivity, and hysteresis makes an intervention persistent.

A routing controller $\phi_r$ maps client $i$'s observable history $H_i^t$ to a distribution over cluster routes:
\begin{equation}
\pi_i^t=\phi_r(H_i^t)\in\Delta_K.
\label{eq:controller-map}
\end{equation}
Let $k_i^{(1)}(t)$ and $k_i^{(2)}(t)$ denote client $i$'s lowest and second lowest loss clusters. The Greedy controller is
\begin{equation}
\phi_{\rm gr}(H_i^t)
=
e_{k_i^{(1)}(t)}.
\label{eq:greedy-controller}
\end{equation}
It exploits the current loss ordering by routing only to the lowest loss model. Once distinct useful models have formed, this promotes selective group specific updates. Under collapsed initialization, however, it gives the unused model no exposure and leaves its activity equal to zero.

The Rank Hedge controller is
\begin{equation}
\phi_{\rm rh}(H_i^t)
=
0.8e_{k_i^{(1)}(t)}+0.2e_{k_i^{(2)}(t)}.
\label{eq:rank-hedge-controller}
\end{equation}
Its second ranked weight explores an alternative cluster and can activate a model that Greedy starves. This exposure does not by itself ensure selectivity: if incompatible clients are redirected to the same model, their contributions increase $\|c_g^t\|_2$ in Theorem~\ref{thm:selective_exposure} and can oppose recovery.

The controller distribution is not used as a soft aggregation weight. Client $i$ samples one cluster and sends a one hot executed route,
\begin{equation}
a_i^t\sim\pi_i^t,
\qquad
q_i^t=e_{a_i^t},
\label{eq:hard-sampled-route}
\end{equation}
then trains only $w_{a_i^t}^t$. The feedback assigned to the selected controller arm is the validation loss expected over this hard cluster draw,
\begin{equation}
z_i^t
=\sum_{k=1}^K\pi_{ik}^t\widehat L_i(w_k^t).
\label{eq:expected-hard-feedback}
\end{equation}
Thus controller feedback averages the losses of possible hard routes, whereas the server sees only the sampled route $q_i^t$ and the corresponding local model update.

For $S_{t,k}=\{i:a_i^t=k\}$ and $N_{t,k}=\sum_{i\in S_{t,k}}n_i$, the server performs hard sample weighted aggregation:
\begin{equation}
w_k^{t+1}
=
\begin{cases}
\displaystyle\frac1{N_{t,k}}\sum_{i\in S_{t,k}}n_i\widetilde w_i^t,
&N_{t,k}>0,\\[0.6em]
w_k^t,&N_{t,k}=0.
\end{cases}
\label{eq:hard-server-aggregation}
\end{equation}
This is the empirical counterpart of \eqref{eq:routing_controlled_dynamics}: the sampled one hot routes determine $M_{gk}^t$, $s_k^t$, and $\beta_{gk}^t$, while $\pi_i^t$ never enters the server average.

Because the useful balance between exposure and selectivity depends on the unknown routing geometry and changes with the cluster models, each client independently runs EXP3 over a small controller menu \citep{doi:10.1137/S0097539701398375}. Outside FEMNIST, $\Pi=\{\phi_{\rm gr},\phi_{\rm rh}\}$. At block $b$, client $i$ forms $p_i^b$, samples $R_i^b\sim p_i^b$, holds that controller throughout the block, sets $\ell_i^b=|\mathcal T_b|^{-1}\sum_{t\in\mathcal T_b}z_i^t$, and updates
\begin{equation}
\begin{aligned}
p_{ir}^b
&=(1-\gamma)\frac{\omega_{ir}^b}{\sum_s\omega_{is}^b}
+\frac{\gamma}{|\Pi|},
\\
\omega_{ir}^{b+1}
&=\omega_{ir}^b
\exp\!\left(
-\eta_{\rm pol}
\frac{\mathbf 1\{R_i^b=r\}\ell_i^b}{p_{ir}^b}
\right).
\end{aligned}
\label{eq:blocked-exp3}
\end{equation}
Blocking gives a selected intervention time to alter the shared models before its feedback is evaluated. The exploratory component of \eqref{eq:blocked-exp3} maintains opportunities for exposure, while learned controller weights can shift probability toward Greedy when the recovered models make selective exploitation effective. Since controller choices alter future model states and losses, ordinary EXP3 regret does not guarantee recovery \citep{arora2012onlinebanditlearningadaptive}. Theorem~\ref{thm:selective_exposure} instead evaluates whether the realized routing sequence supplies sufficient exposure and persistence while controlling incompatible forcing.

Algorithm~\ref{alg:blocked_controller_selection} gives the complete training order. At the start of each block, every client samples one controller arm and retains it throughout the block. In each communication round, the server broadcasts the current cluster models, each client samples one hard cluster route and trains only that model, and the server aggregates the returned models by sampled cluster using client sample counts. At the end of the block, each client updates EXP3 using its average expected hard route feedback.

\begin{routedalgorithm}[H]
\small
\caption{Blocked decentralized routing control.}
\label{alg:blocked_controller_selection}
\vspace{-0.4em}
\hrule
\vspace{0.25em}
\DontPrintSemicolon
\SetAlgoLined
\textbf{Server executes:} initialize $W^0$ and $\omega_{ir}^0\gets1$ for every client $i$ and controller $r\in\Pi$;
\For{blocks $b=0,1,\ldots$}{
  \ForEach{client $i$ in parallel}{
    compute $p_i^b$ from \eqref{eq:blocked-exp3}; sample $R_i^b\sim p_i^b$;
  }
  \For{$t\in\mathcal T_b$}{
    broadcast $W^t$;
    \ForEach{client $i$ in parallel}{
      $(a_i^t,q_i^t,n_i,\widetilde w_i^t,z_i^t)\gets\textsc{ClientUpdate}(i,W^t,R_i^b,H_i^t)$;
    }
    \ForEach{cluster $k$}{
      $S_{t,k}\gets\{i:a_i^t=k\}$ and $N_{t,k}\gets\sum_{i\in S_{t,k}}n_i$;
      update $w_k^{t+1}$ by \eqref{eq:hard-server-aggregation};
    }
  }
  \ForEach{client $i$ in parallel}{
    $\ell_i^b\gets|\mathcal T_b|^{-1}\sum_{t\in\mathcal T_b}z_i^t$;
    update $\{\omega_{ir}^{b+1}:r\in\Pi\}$ by \eqref{eq:blocked-exp3};
  }
}
\textsc{ClientUpdate}$(i,W^t,R_i^b,H_i^t)$:
$\pi_i^t\gets\phi_{R_i^b}(H_i^t)$; sample $a_i^t\sim\pi_i^t$ and set $q_i^t\gets e_{a_i^t}$;
$\widetilde w_i^t\gets\operatorname{LocalUpdate}_i(w_{a_i^t}^t)$;
$z_i^t\gets\sum_{k=1}^K\pi_{ik}^t\widehat L_i(w_k^t)$;
\Return $(a_i^t,q_i^t,n_i,\widetilde w_i^t,z_i^t)$;
\vspace{0.15em}
\hrule
\vspace{0.3em}
\end{routedalgorithm}

With a slowly trained neural model, isolated exploration may create activity without providing enough persistent compatible updates for the model to become competitive. The selective hysteresis controller therefore treats an entered target $c_i$ as a temporary commitment. Let $d_i^t$ count client $i$'s routed updates to $c_i$, and define the target's excess loss
\begin{equation}
G_i^t
:=
\widehat L_i(w_{c_i}^t)
-
\min_j\widehat L_i(w_j^t).
\label{eq:hysteresis-gap}
\end{equation}
The controller is
\begin{equation}
\phi_{\rm hys}(H_i^t)
=
\begin{cases}
e_{c_i}, & d_i^t<H_{\min},\\
e_{c_i}, & H_{\min}\le d_i^t<H_{\max}\ \text{and}\ G_i^t>\epsilon_{\rm exit},\\
e_{k_i^{(1)}(t)}, & \text{otherwise}.
\end{cases}
\label{eq:hysteresis-controller}
\end{equation}
Thus an entered client remains on the target for at least $H_{\min}$ updates and subsequently continues while the target's excess loss exceeds $\epsilon_{\rm exit}$, subject to the episode-specific cap $H_{\max}$ \citep{HESPANHA2003263}. The counter records completed target updates before the current routing decision. When the competitiveness or cap condition releases an episode, the client executes its Greedy route in that same round and becomes eligible for a new controller decision in the following round. The cap applies separately to every hysteresis episode, including episodes entered after EXP3 selection. Hysteresis supplies persistence, but does not itself ensure compatible update composition.

We address that entry problem in FEMNIST with one synchronized round. Each client first reports a provisional Greedy route $\widetilde a_i^0$, which is used only to measure the sample weighted provisional mass
\begin{equation}
\widetilde\mu_k^0
:=
\frac1N\sum_i n_i\mathbf1\{\widetilde a_i^0=k\}.
\label{eq:provisional-mass}
\end{equation}
After the server broadcasts these masses, client $i$ tests a candidate underused cluster $k$ by making a temporary local update $U_i(w_k^0)$ and computing
\begin{equation}
B_{ik}^0
:=
\min_j\widehat L_i(w_j^0)
-\widehat L_i\!\left(U_i(w_k^0)\right).
\label{eq:entry-benefit}
\end{equation}
The provisional routes are never aggregated. A client executes the candidate route only when $\widetilde\mu_k^0<\mu_{\rm rev}$ and $B_{ik}^0>\tau$. Qualifying clients therefore enter the weak model in the same round, creating coordinated exposure from clients whose local tests indicate compatibility.

Algorithm~\ref{alg:femnist-coordinated-entry} mirrors Algorithm~\ref{alg:blocked_controller_selection}, but inserts one synchronized entry round. Provisional Greedy reports determine only the broadcast masses; \textsc{EntryUpdate} either executes a qualifying underused cluster route and enters forced hysteresis or executes the provisional Greedy route. The server aggregates these executed round zero updates using the same hard update as Algorithm~\ref{alg:blocked_controller_selection}. Forced entrants make no EXP3 weight updates during their entry episode. Upon release they execute Greedy, then become eligible for independent selection over $\Pi_{\rm FEM}=\{\phi_{\rm gr},\phi_{\rm hys}\}$ in the following round. FEMNIST uses client-specific actions rather than shared fixed blocks: a Greedy action lasts one round, whereas an entered hysteresis action persists until release. EXP3 receives the mean normalized training-loss feedback of each completed selected action, as specified in Appendix~\ref{app:femnist-supplement}. Thus entry supplies exposure and initial selectivity, while hysteresis supplies persistence. Further settings and interpretation appear in Appendix~\ref{app:femnist-supplement}.

\begin{routedalgorithm}[H]
\footnotesize
\caption{Synchronized entry followed by decentralized routing control.}
\label{alg:femnist-coordinated-entry}
\vspace{-0.4em}
\hrule
\vspace{0.25em}
\DontPrintSemicolon
\SetAlgoLined
\textbf{Server executes synchronized entry:}
\ForEach{client $i$ in parallel}{
  report $\widetilde a_i^0\in\arg\min_j\widehat L_i(w_j^0)$;
}
compute \eqref{eq:provisional-mass} and broadcast $\widetilde\mu^0$; these reports do not update any model;
\ForEach{client $i$ in parallel}{
  $(a_i^0,q_i^0,n_i,\widetilde w_i^0,F_i,c_i,d_i)\gets\textsc{EntryUpdate}(i,W^0,\widetilde a_i^0,\widetilde\mu^0)$;
}
\ForEach{cluster $k$}{
  $S_{0,k}\gets\{i:a_i^0=k\}$ and $N_{0,k}\gets\sum_{i\in S_{0,k}}n_i$;
  update $w_k^1$ by \eqref{eq:hard-server-aggregation};
}
\textsc{EntryUpdate}$(i,W^0,\widetilde a_i^0,\widetilde\mu^0)$:\par
\begingroup\leftskip=1em
$\mathcal C_i^0\gets\{k:\widetilde\mu_k^0<\mu_{\rm rev}\}$;\par
\textbf{if} $\mathcal C_i^0\ne\varnothing$, compute \eqref{eq:entry-benefit} and choose $c_i^0\in\arg\max_{k\in\mathcal C_i^0}B_{ik}^0$;\par
\textbf{if} $\mathcal C_i^0\ne\varnothing$ and $B_{i,c_i^0}^0>\tau$ \textbf{then}\par
\begingroup\leftskip=1em
$a_i^0\gets c_i^0$; $\widetilde w_i^0\gets U_i(w_{c_i^0}^0)$; $(F_i,c_i,d_i)\gets(1,c_i^0,1)$;\par
\endgroup
\textbf{else}\par
\begingroup\leftskip=1em
$a_i^0\gets\widetilde a_i^0$; $\widetilde w_i^0\gets\operatorname{LocalUpdate}_i(w_{a_i^0}^0)$; $F_i\gets0$;\par
\endgroup
$q_i^0\gets e_{a_i^0}$; \Return $(a_i^0,q_i^0,n_i,\widetilde w_i^0,F_i,c_i,d_i)$;\par
\endgroup

\textbf{Server executes subsequent rounds:}
\For{rounds $t=1,\ldots,T-1$}{
  broadcast $W^t$;\par
  \ForEach{client $i$ in parallel}{
    if an episode is active, apply \eqref{eq:hysteresis-controller}; on release, execute Greedy in this round, close the episode, and set $F_i\gets0$;\par
    otherwise, sample a new arm from $\Pi_{\rm FEM}$ using \eqref{eq:blocked-exp3}, with a client-specific action index;\par
    execute the selected Greedy route, hysteresis fallback, or committed route, and set $q_i^t\gets e_{a_i^t}$;\par
    $\widetilde w_i^t\gets\operatorname{LocalUpdate}_i(w_{a_i^t}^t)$ and compute $z_{i,\rm FEM}^t$ by \eqref{eq:femnist-normalized-feedback};\par
    accumulate feedback only for rounds belonging to the selected action; the Greedy exit round is excluded from the released episode;\par
    when an EXP3-selected action completes, update its weights by \eqref{eq:blocked-exp3} using its mean accumulated cost and its sampling probability; do not update EXP3 for the forced entry episode;\par
    an episode released this round is eligible for a new controller decision only at round $t+1$;\par
  }
  \ForEach{cluster $k$}{
    $S_{t,k}\gets\{i:a_i^t=k\}$ and $N_{t,k}\gets\sum_{i\in S_{t,k}}n_i$;
    update $w_k^{t+1}$ by \eqref{eq:hard-server-aggregation};
  }
}
\vspace{0.15em}
\hrule
\vspace{0.3em}
\end{routedalgorithm}

\section{Experiments}
\label{sec:experiments}
We evaluate decentralized online controller selection against Greedy IFCA and FedCAM on Linear regression and transformed FEMNIST; multiclass GMM and California Housing appear in Appendices~\ref{app:gmm-supplement} and~\ref{app:housing-supplement}. Because these predictors have different parameter spaces, their raw initialization scales are not comparable. We instead summarize the induced routing geometry by the round zero attraction $\mathbf A_{gk}^0:=|I_g|^{-1}\sum_{i\in I_g}\mathbf1\{k=\arg\min_j\widehat L_i(w_j^0)\}$. A collapsed entrance sends several groups to the same cluster and leaves another unused; we test whether adaptive routing can transform such an entrance into models that remain separated under controller free greedy evaluation.

\paragraph{Code and reproducibility.}
The complete code, tests, compact evidence, plotting scripts, and experiment runners used to reproduce the results and figures are available at \url{https://github.com/prsinha1/FedRecover}. The experiments reported in this paper correspond to release \texttt{v1.0.1}. The complete reproducibility package is the asset \path{FedRecover_reproducibility_v1.0.1.zip} attached to \href{https://github.com/prsinha1/FedRecover/releases/tag/v1.0.1}{release \texttt{v1.0.1}}. GitHub's automatically generated ``Source code'' archives should not be used.

\paragraph{Linear Regression:}
\label{sec:exp-linear}
We use $80$ clients balanced across the two linear groups of Section~\ref{sec:routing_recovery} and initialize $w_k^0=\bar\theta+\rho z_k$, with $z_k\sim\mathcal N(0,I)$, so $\rho$ controls spread around the pooled parameter. The audited common preference frequency increases from $0.10$ to $1.00$ across the tested $\rho$ sweep. Using the recovery error $R_T$ from \eqref{eq:recovery-error}, Figure~\ref{fig:main-results}(a) reports the recovery rate, $\widehat{\Pr}(R_T<0.5)$. Greedy IFCA and FedCAM become less reliable as $\rho$ increases, while EXP3 recovers every evaluated Greedy failure cell; the dashed curve is the certificate based upper bound on Greedy recovery.

\begin{figure}[H]
    \centering
    \begin{minipage}[c]{0.235\textwidth}\centering
        \includegraphics[width=\linewidth]{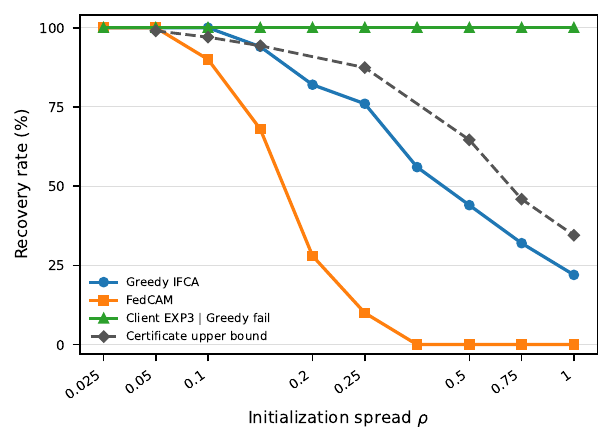}
        {\small\textbf{(a) Recovery sweep}}
    \end{minipage}\hfill
    \begin{minipage}[c]{0.235\textwidth}\centering
        \includegraphics[width=\linewidth]{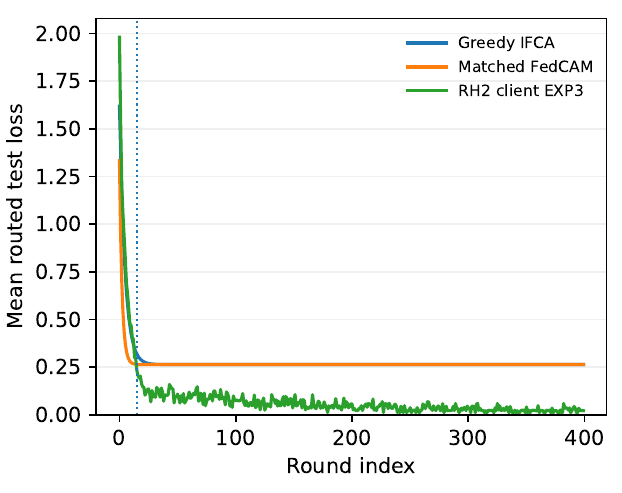}
        {\small\textbf{(b) Routed test loss}}
    \end{minipage}\hfill
    \begin{minipage}[c]{0.235\textwidth}\centering
        \includegraphics[width=\linewidth]{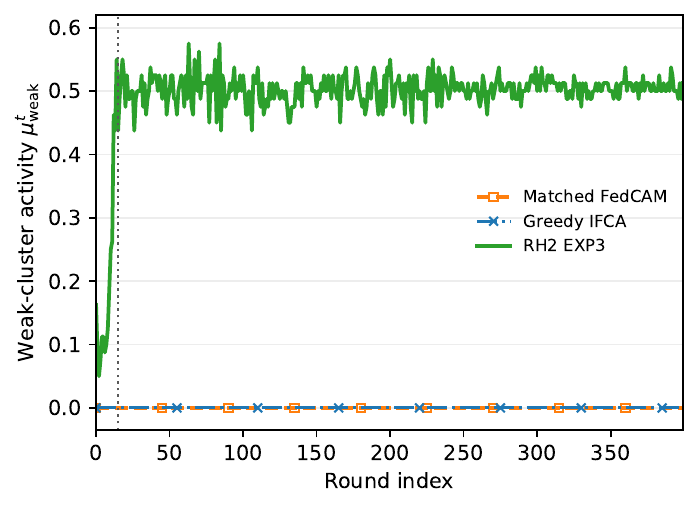}
        {\small\textbf{(c) Weak cluster activity}}
    \end{minipage}\hfill
    \begin{minipage}[c]{0.235\textwidth}\centering
        \includegraphics[width=\linewidth]{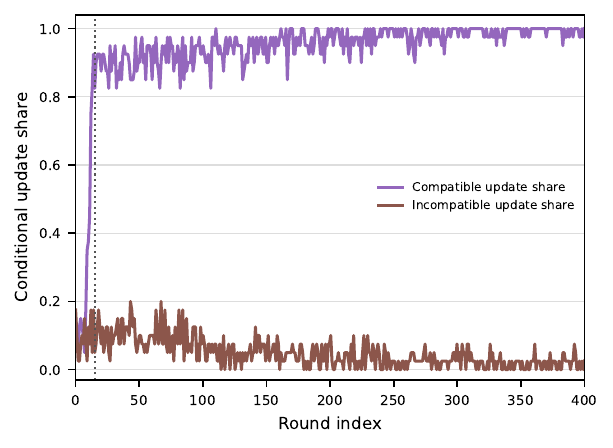}
        {\small\textbf{(d) Update composition}}
    \end{minipage}
    \par\smallskip
    \begin{minipage}[c]{0.235\textwidth}\centering
        \includegraphics[width=\linewidth]{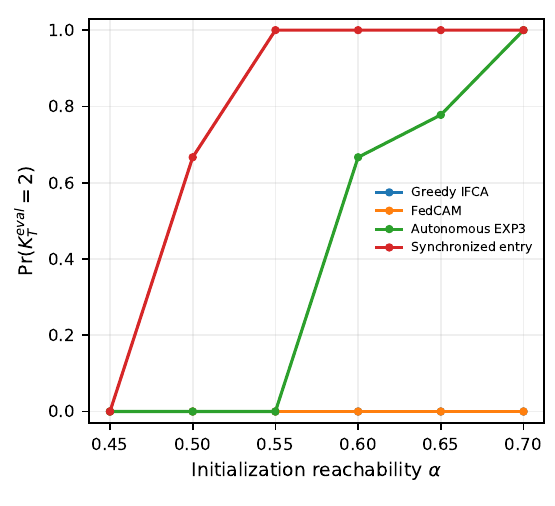}
        {\small\textbf{(e) Usability across $\alpha$}}
    \end{minipage}\hfill
    \begin{minipage}[c]{0.235\textwidth}\centering
        \includegraphics[width=\linewidth]{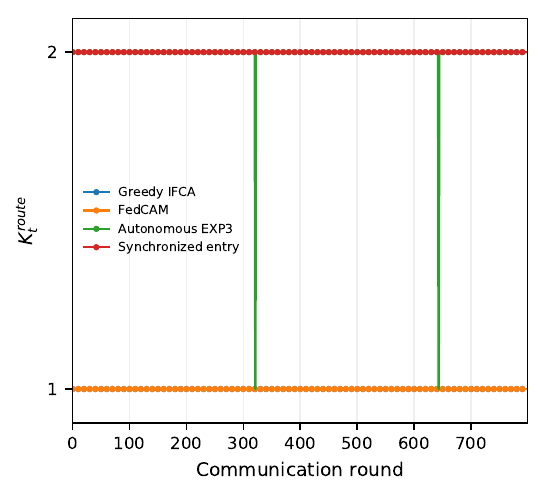}
        {\small\textbf{(f) Executed activity}}
    \end{minipage}\hfill
    \begin{minipage}[c]{0.235\textwidth}\centering
        \includegraphics[width=\linewidth]{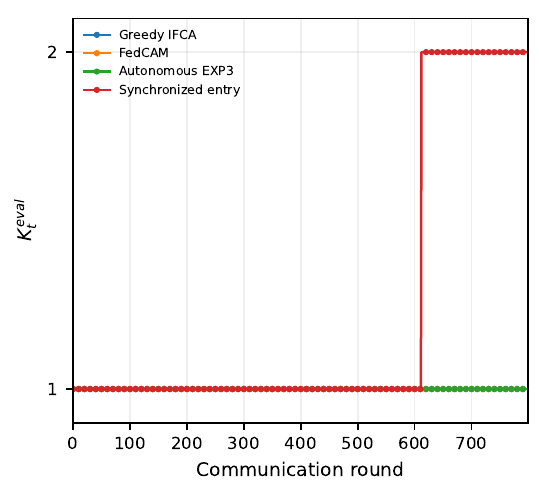}
        {\small\textbf{(g) Model usability}}
    \end{minipage}\hfill
    \begin{minipage}[c]{0.235\textwidth}\centering
        \includegraphics[width=\linewidth]{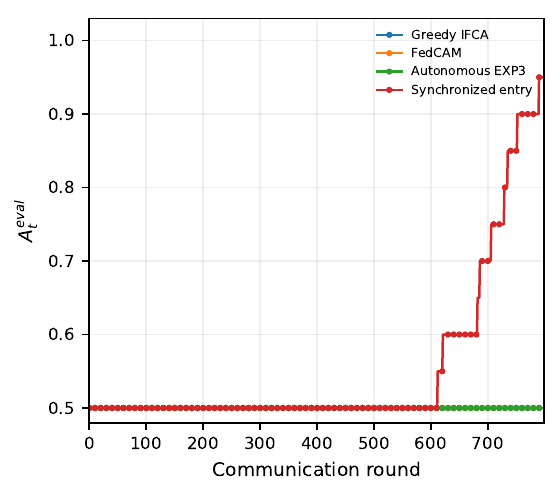}
        {\small\textbf{(h) Specialization}}
    \end{minipage}
    \caption{\textbf{Recovery from collapsed initialization} in linear (a:d) and FEMNIST (e:h) settings.}
    \label{fig:main-results}
    \label{fig:linear-main}
    \label{fig:femnist-main}
\end{figure}

On a representative failed start, panel (b) plots routed test loss $L_t^{\rm route}:=n^{-1}\sum_iL_{i,{\rm test}}(w_{r_i^t}^t)$. Panel (c) shows that only EXP3 restores weak cluster mass $\mu_{\rm weak}^t:=\sum_g\alpha_gM_{g,{\rm weak}}^t$, and hence weak cluster activity $s_{\rm weak}^t=\mathbf1\{\mu_{\rm weak}^t>0\}$. Panel (d) shows that its updates become dominated by the compatible group, with $\beta_{g^\star,{\rm weak}}^t\to1$ and $\sum_{h\ne g^\star}\beta_{h,{\rm weak}}^t\to0$. Thus EXP3 creates the activity and compatible update composition identified by Equation~\eqref{eq:matched_error_dynamics}.

\paragraph{FEMNIST Classification:}
\label{sec:exp-femnist}
We use $10$ clean clients and $10$ clients with digits $6$ and $9$ rotated by $180^\circ$ without relabeling. Each model is a width $64$ ReLU network with one hidden layer. We set $w_0^0=w_{\rm mix}$ and $w_1^0=(1-\alpha)\xi+\alpha w_{\rm trans}$, where $w_{\rm mix}$ and $w_{\rm trans}$ are pretrained on pooled and transformed clients and $\xi$ is a weak base model. Increasing $\alpha$ moves the unused model toward $w_{\rm trans}$, yet $\mathbf A^0=\left(\begin{smallmatrix}1&0\\1&0\end{smallmatrix}\right)$ for every tested $\alpha$. All four methods run for $800$ communication rounds. The synchronized-entry configuration uses $H_{\max}=800$, while autonomous EXP3 uses $H_{\max}=160$; full controller settings appear in Appendix~\ref{app:femnist-supplement}.

At $W^{800}$, each client selects its lowest held-out-loss model, $r_i^{800,\mathrm{eval}}\in\arg\min_k\widehat L_{i,\mathrm{te}}(w_k^{800})$, and $K_{800}^{\mathrm{eval}}:=|\{r_i^{800,\mathrm{eval}}:i\in[n]\}|$ counts the selected models. Panel (e) reports terminal two-model usability: Greedy IFCA and FedCAM remain at zero, autonomous EXP3 improves with $\alpha$, and synchronized entry reaches a recovery fraction of one from $\alpha=0.55$ onward. At $W^{800}$, $42$ of the $54$ synchronized-entry configurations have two usable models on both training and held-out data. Panels (f), (g), and (h) show the representative run at $\alpha=0.60$, data seed $1$ and selector seed $0$. Panel (f) reports $K_t^{\mathrm{route}}$: synchronized entry routes updates to both models throughout the $800$ rounds, autonomous EXP3 does so in $798$ rounds, and the baselines use one model. Panel (g) reports training-loss-based, controller-free model usability: autonomous EXP3 remains at one usable model, whereas synchronized entry reaches two at $W^{612}$ and retains both through the plotted horizon. Panel (h) reports permutation-invariant assignment accuracy, $A_t^{\mathrm{eval}}:=\max_{m\in\mathfrak S_2}\frac12\sum_g|I_g|^{-1}\sum_{i\in I_g}\mathbf1\{r_i^{t,\mathrm{eval}}=m(g)\}$. Synchronized entry raises this score from $1/2$ to $0.95$ at the final plotted state $W^{799}$, while autonomous EXP3 remains at $1/2$. Thus executed activity, two-model usability, and exact group specialization are distinct: synchronized entry with persistent commitment produces two usable models in this run, but does not achieve perfect group specialization within the evaluated horizon.

\bibliographystyle{plainnat}
\bibliography{refs}

\clearpage
\appendix
\section{Theory Supplement}
\label{app:theory}

This appendix proves the theoretical results stated in the main paper.

\subsection{Favorable greedy specialization}
\label{app:proof_greedy_specialization}

\begin{proof}[Proof of Theorem~\ref{thm:specialist_basin}]
Choose a matching $m\in\mathfrak S_2$ witnessing $W^0\in\mathcal B_{\rm specialize}$ and relabel the clusters so that $m(g)=g$. Let $r_g=\|w_g^0-\theta_g\|_2$. Then $r_0+r_1<\Delta$.
For group $0$, the reverse triangle inequality gives
\[
\|w_1^0-\theta_0\|_2
\ge \Delta-r_1
>r_0
=\|w_0^0-\theta_0\|_2.
\]
The analogous inequality holds for group $1$, so the initial greedy routing is separated. If routing remains separated through round $t$, cluster $g$ receives only the group $g$ gradient, and therefore
\[
w_g^{t+1}-\theta_g
=
(1-\eta)(w_g^t-\theta_g),
\qquad
\|w_g^t-\theta_g\|_2
=
(1-\eta)^tr_g.
\]
At every round,
\[
\|w_{1-g}^t-\theta_g\|_2
\ge
\Delta-\|w_{1-g}^t-\theta_{1-g}\|_2
\ge
\Delta-r_{1-g}
>
r_g
\ge
\|w_g^t-\theta_g\|_2.
\]
Thus the separated preference is preserved by induction. Since $(1-\eta)^t\to0$, each cluster converges to its matched specialist.
\end{proof}

\subsection{Proof of greedy starvation}
\label{app:proof_greedy_starvation}

\begin{proof}[Proof of Theorem~\ref{thm:starvation_basin}]
It suffices to prove the result for $j=1$; the case $j=0$ follows by exchanging the cluster labels. Membership in $\mathcal B_{{\rm starve},1}$ implies
$L_g(w_0^0)<L_g(w_1^0)$ for both groups. Hence both groups initially select cluster $0$, while cluster $1$ receives no update.

Suppose this common preference has persisted through round $t$. Then $w_1^t=w_1^0$, while the active cluster follows
\[
w_0^{t+1}
=
w_0^t-\eta(w_0^t-\bar\theta_\alpha),
\qquad
w_0^t
=
\bar\theta_\alpha
+
(1-\eta)^t(w_0^0-\bar\theta_\alpha).
\]
Thus $w_0^t$ lies on the segment joining $w_0^0$ and $\bar\theta_\alpha$. Convexity of $L_g$ gives
\[
L_g(w_0^t)
\le
\max\{L_g(w_0^0),L_g(\bar\theta_\alpha)\}
<
L_g(w_1^0)
=
L_g(w_1^t).
\]
Both groups therefore continue to select cluster $0$. Induction proves that $w_1^t=w_1^0$ for every $t$, and the explicit recursion gives $w_0^t\to\bar\theta_\alpha$.
\end{proof}

\subsection{Finite horizon recovery under arbitrary routing}
\label{app:proof_finite_horizon_recovery}

Fix a matching $m$ and recall that
$e_g^t=w_{m(g)}^t-\theta_g$,
$a_g^t=s_{m(g)}^t$, and
$c_g^t=\sum_{h\ne g}\beta_{h,m(g)}^t(\theta_h-\theta_g)$.
Recall $E_g(t,T)=\sum_{\tau=t}^{T-1}a_g^\tau$, the number of active updates received by the matched cluster.
Repeated substitution in Equation~\eqref{eq:matched_error_dynamics} gives the exact response
\begin{equation}
e_g^T
=
(1-\eta)^{E_g(0,T)}e_g^0
+
\eta
\sum_{t=0}^{T-1}
a_g^t
(1-\eta)^{E_g(t+1,T)}
c_g^t.
\label{eq:finite_horizon_response}
\end{equation}

\begin{proof}[Proof of Theorem~\ref{thm:selective_exposure}]
Equation~\eqref{eq:finite_horizon_response} follows by repeatedly applying Equation~\eqref{eq:matched_error_dynamics}. Each active round multiplies the existing matched error by $1-\eta$ and adds $\eta c_g^t$, while each inactive round leaves the error unchanged.

Taking norms in Equation~\eqref{eq:finite_horizon_response} and applying the triangle inequality gives
\[
\|e_g^T\|_2
\le
(1-\eta)^{E_g(0,T)}
\|e_g^0\|_2
+
\eta
\sum_{t=0}^{T-1}
a_g^t
(1-\eta)^{E_g(t+1,T)}
\|c_g^t\|_2.
\]
Condition~\eqref{eq:selective_exposure_condition} therefore implies
$\|w_{m(g)}^T-\theta_g\|_2=\|e_g^T\|_2\le r$.
The excess risk conclusion follows from
\[
L_g(w_{m(g)}^T)-L_g(\theta_g)
=
\frac12\|w_{m(g)}^T-\theta_g\|_2^2.
\]

Finally, if $r<\Delta/2$, then
\[
\|w_{m(1-g)}^T-\theta_g\|_2
\ge
\Delta-\|w_{m(1-g)}^T-\theta_{1-g}\|_2
\ge
\Delta-r
>
r
\ge
\|w_{m(g)}^T-\theta_g\|_2.
\]
Thus each group strictly prefers its matched model under terminal greedy routing.
\end{proof}

Equation~\eqref{eq:selective_exposure_condition} separates the two routing requirements. The first term contracts only when the matched cluster receives repeated updates. The second term measures incompatible contributions, discounted by later active updates. Their combination is the persistent selective exposure condition used in the main paper.

\clearpage
\section{Experimental Supplement}
\label{app:experimental-supplement}
\setcounter{figure}{0}
\renewcommand{\thefigure}{B.\arabic{figure}}
\renewcommand{\theHfigure}{B.\arabic{figure}}

\subsection{Shared evaluation and decentralized controller selection}
\label{app:experimental-setup}

Within each experimental cell, all methods use the same clients, corresponding initial cluster models, and local optimization budget. The FEMNIST horizon and evaluation-state conventions are specified below. We record executed routes during training and controller free greedy routes after freezing the evaluated models. Linear and GMM use $\{\phi_{\rm gr},\phi_{\rm rh}\}$, Housing uses $\{\phi_{\rm gr},\phi_{\rm com},\phi_{\rm rev}\}$, and autonomous FEMNIST uses $\{\phi_{\rm gr},\phi_{\rm hys}\}$. In the fixed-block experiments, a selected arm remains fixed within each block $\mathcal T_b$, with feedback averaged over that block. FEMNIST instead uses client-specific action durations and normalized training-loss feedback. Algorithm~\ref{alg:blocked_controller_selection} specifies fixed-block client controller selection, hard route sampling, local training, sample weighted server aggregation, and the block level EXP3 update; Algorithm~\ref{alg:femnist-coordinated-entry} specifies the FEMNIST action timing.

\subsection{Linear supplementary evidence}
\label{app:linear-supplement}

Under the pooled centered initialization of Section~4, the audited fraction of initializations for which both groups prefer the same cluster increases from $0.10$ to $1.00$ across the tested $\rho$ sweep. Figure~B.1 asks whether the online selector maintains a fixed mixture of Greedy and Rank Hedge or adapts controller use as training evolves. For each controller $r$, it reports the client averaged EXP3 probability $\bar p_r^b=n^{-1}\sum_i p_{ir}^b$ and the executed selection fraction $n^{-1}\sum_i\mathbf{1}\{R_i^b=r\}$, smoothed over five blocks and averaged across $147$ Greedy failure cells. EXP3 begins with approximately equal probability on both controllers, so Rank Hedge initially supplies alternative routing. The median cell recovers by block $1$. After recovery, the mean Greedy probability rises steadily from approximately $0.5$ to $0.9$, while the Rank Hedge probability falls to approximately $0.1$. The executed selection fractions closely track these learned probabilities. Thus the online learner does not maintain a fixed exploration rate: once the recovered models make lowest loss routing effective, it increasingly selects Greedy and suppresses unnecessary redirection. Because most of this adaptation occurs after the median recovery block, the figure demonstrates adjustment after recovery rather than showing that EXP3 learned the escaping controller before recovery.

\begin{figure}[H]
    \centering
    \includegraphics[width=0.52\linewidth]{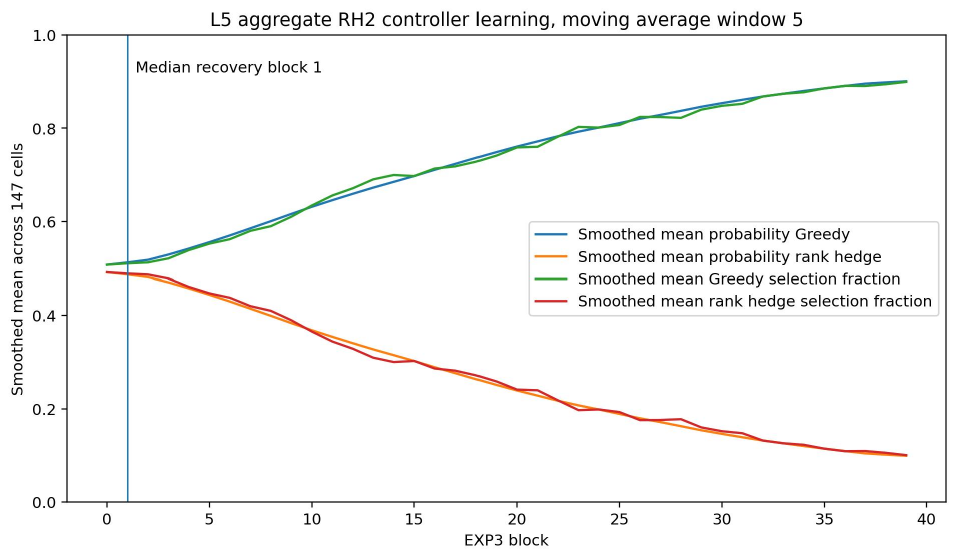}
    \caption{Linear controller adaptation over Greedy failure cells.}
    \label{fig:app-linear-controller}
\end{figure}
\FloatBarrier

\subsection{FEMNIST synchronized-entry specialist construction}
\label{app:femnist-supplement}

The round zero audit covers every $\alpha\in\{.45,.50,.55,.60,.65,.70\}$ and data seed in $\{0,1,2\}$. In all $18$ cells, both groups prefer cluster $0$, so $\mathbf A^0=\begin{psmallmatrix}1&0\\1&0\end{psmallmatrix}$, $K_0=1$, and the common preference frequency is $1.00$.

All four FEMNIST methods are evaluated for $T=800$ rounds, indexed $t=0,\ldots,799$, producing terminal models $W^{800}$. The initialization construction is given in Section~\ref{sec:exp-femnist}. Model parameters use float32 arithmetic, and each executed client update uses one local SGD step with learning rate $0.003$. We evaluate six initialization values, $\alpha\in\{0.45,0.50,0.55,0.60,0.65,0.70\}$, and three data seeds. Each adaptive method additionally uses three selector seeds, giving $54$ runs per adaptive method and $18$ runs per baseline, or $144$ runs overall. The nine adaptive runs at each $\alpha$ comprise three data contexts with three selector seeds each, rather than nine independent datasets.

Both adaptive methods use $H_{\min}=5$, $\epsilon_{\rm exit}=0.005$, $\mu_{\rm rev}=0.20$, $\tau=-0.7$, and EXP3 parameters $\eta_{\rm EXP3}=2.0$ and $\gamma=0.001$. The synchronized-entry method uses $H_{\max}=800$, whereas autonomous EXP3 uses $H_{\max}=160$. Consequently, this comparison evaluates these controller configurations and does not isolate synchronization under matched commitment caps.

Selective hysteresis acts only after client $i$ has entered an episode with target $c_i$; it does not choose the target. Equations~\eqref{eq:hysteresis-gap} and~\eqref{eq:hysteresis-controller} give the routing rule. The entry update initializes the target-update counter to one. At each subsequent routing decision the counter records completed target updates. After the minimum commitment period, an episode ends when $G_i^t\le\epsilon_{\rm exit}$ or its own counter reaches $H_{\max}$. The client executes Greedy in that same exit round and becomes eligible for a new controller decision in the following round. These rules apply both to the synchronized-entry episode and to later EXP3-selected episodes.

Since training executes rounds $0,\ldots,799$, an episode beginning at round zero cannot reach an age-$800$ cap-triggered exit within the observed horizon. Competitiveness-triggered exits can occur earlier. Therefore, successful cap-$800$ recovery does not by itself demonstrate successful recovery after a forced cap-triggered handoff to EXP3.

Algorithm~\ref{alg:femnist-coordinated-entry} coordinates only the initial entry. Provisional Greedy choices measure model use but never train a model. The server broadcasts their sample weighted masses, revealing underused candidates. Each client temporarily updates a candidate and evaluates the benefit in \eqref{eq:entry-benefit}. Qualifying clients execute the candidate route together in round $0$ and initialize $(c_i,d_i)$; the others execute their provisional Greedy routes. During the forced hysteresis episode, the client accumulates target updates but does not update EXP3. Upon release it executes Greedy, then becomes eligible for decentralized controller selection in the following round.

A Greedy action lasts one round. Selecting hysteresis either produces a one-round Greedy fallback or begins an episode that persists until its exit condition is met. EXP3 receives feedback only when its selected action completes; unfinished episodes are right-censored at the evaluation horizon. The synchronized entry episode supplies no EXP3 weight update. Define
\begin{equation}
\begin{aligned}
s_i&=\max\!\left\{\min_k\widehat L_{i,\rm tr}(w_k^0),10^{-12}\right\},\\
z_{i,\rm FEM}^t&=\frac1{10}\operatorname{clip}\!\left(
\frac{\widehat L_{i,\rm tr}(w_{a_i^t}^t)}{s_i},0,10\right).
\end{aligned}
\label{eq:femnist-normalized-feedback}
\end{equation}
The completed action's feedback is the mean of these costs over its executed rounds, excluding the Greedy exit round of a released hysteresis episode. The EXP3 update uses the probability with which that action was sampled. Held-out losses are used for evaluation, not controller selection or local optimization.

Panel (e) reports the fraction of runs with two models selected by controller-free, lowest-held-out-loss evaluation at $W^{800}$. Panels (f)--(h) show the pre-update trajectory $W^0,\ldots,W^{799}$: executed routing activity in panel (f), and training-loss-based controller-free usability and assignment accuracy in panels (g) and (h). FedCAM evaluation uses its combined global-plus-cluster predictor.

Training recovery means $K_{\rm tr}(W^{800})=2$. Joint training and held-out recovery requires $K_{\rm tr}(W^{800})=K_{\rm te}(W^{800})=2$. Sustained recovery requires this joint condition at every state $W^{781},\ldots,W^{800}$. Exact specialization additionally requires permutation-invariant assignment accuracy equal to one on both splits, with the maximizing permutation chosen separately for each split. These assignment scores describe recovery of the client groups and are distinct from digit-classification accuracy. Held-out model preferences are evaluated offline and do not determine training routes.

Figure~\ref{fig:app-femnist-cost} reports the last logged executed-routing test loss,
\begin{equation}
L_{799}^{\rm route}=n^{-1}\sum_iL_{i,\rm test}\!\left(w_{r_i^{799}}^{799}\right).
\end{equation}
Curves show means across runs, and shaded regions show one standard deviation. This quantity evaluates the routes actually executed during training, rather than selecting a model by held-out loss. Synchronized entry incurs higher predictive cost than the baselines at smaller $\alpha$, while its mean routed test loss falls below both baselines at $\alpha=0.65$ and $0.70$.

\begin{figure}[H]
    \centering
    \includegraphics[width=0.36\linewidth]{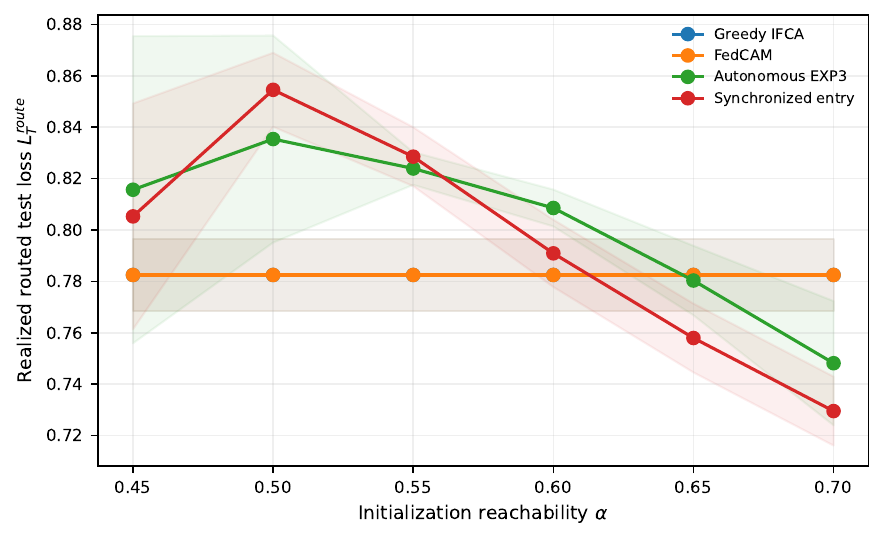}
    \caption{FEMNIST predictive cost under executed routing.}
    \label{fig:app-femnist-cost}
\end{figure}
\FloatBarrier

\subsection{GMM collapse and recovery}
\label{app:gmm-supplement}

We use a multiclass Gaussian mixture model with $G=3$ latent groups, $C=5$ classes, and feature dimension $d=10$. Each group $g\in\{1,2,3\}$ has a feature center $\mu_g\in\mathbb R^d$, and its samples are generated as $x=\mu_g+\sigma\xi,
\qquad
\xi\sim\mathcal N(0,I_d).$
Clients are pure: every client $i$ belongs to one latent group $g_i$, and all of its samples follow the corresponding group distribution. We use ten clients per group.

Labels follow a mismatch model containing shared and group specific components. For class $c\in\{1,\ldots,C\}$ and group $g$, the generating logit and label are
\[
r_{g,c}(x)=\alpha u_c^\top x+\lambda v_{g,c}^\top x+b_{g,c},
\qquad
y=\arg\max_{c\in\{1,\ldots,C\}}r_{g,c}(x).
\]
Here $U\in\mathbb R^{C\times d}$ contains the shared class directions $u_c$, $V_g\in\mathbb R^{C\times d}$ contains the group specific directions $v_{g,c}$, and $b_g\in\mathbb R^C$ is the group specific class bias. Thus clients share part of the classification rule while retaining group dependent decision boundaries. The server trains three linear softmax models for $800$ rounds. The far initialization $W_k^0=R_{\rm far}D_k$
scales fixed, class centered directions $D_k$ of unit Frobenius norm chosen independently of the group specific classifiers.

Figure~\ref{fig:app-gmm} reports the endpoint collapse frequency $\widehat{\Pr}(K_T^{\rm eval}<3)$, permutation invariant terminal assignment $A_T$, and final routed test loss. Greedy IFCA and FedCAM remain vulnerable across the radius sweep, while EXP3 usually restores three group specialization and substantially reduces routed loss on the Greedy collapse cells. The larger variation at $R_{\rm far}=12$ marks a harder recovery regime rather than a monotone initialization severity scale.

\begin{figure}[H]
    \centering
    \includegraphics[width=0.86\linewidth]{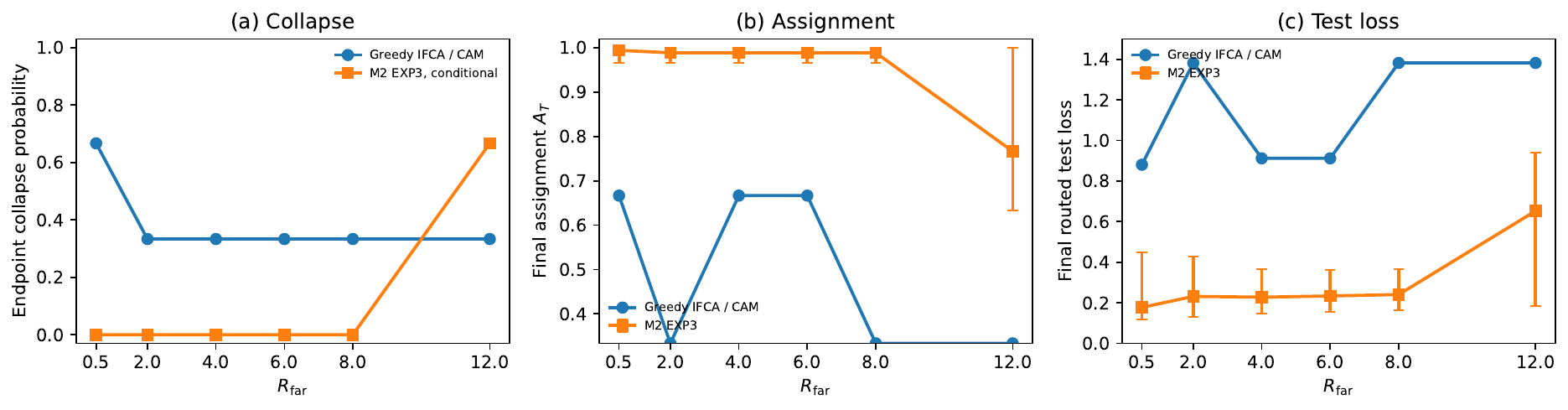}
    \caption{GMM terminal structure and loss across the initialization radius.}
    \label{fig:app-gmm}
\end{figure}
\FloatBarrier

\subsection{Housing collapse and recovery}
\label{app:housing-supplement}

We use California Housing as a real tabular regression benchmark. Each census block group has eight housing and demographic features, with median house value as the target. Let $q_{0.75}^{\rm tr}$ be the global-training $0.75$ quantile of average occupancy per household, \texttt{AveOccup}. We assign a block group to group $0$ if $\texttt{AveOccup}<q_{0.75}^{\rm tr}$ and to group $1$ otherwise. The groups share the feature space and target but represent different occupancy regimes.

Each cluster model is a linear regressor with an intercept. We standardize the eight features using global-training statistics and append a constant coordinate, giving $x\in\mathbb R^9$ and $f_w(x)=x^\top w$. For client $i$, the empirical objective is $L_i(w)=n_i^{-1}\sum_{j=1}^{n_i}(x_{ij}^\top w-y_{ij})^2$. For group oracle parameters $\beta_0^\star$ and $\beta_1^\star$, let $\Delta_\beta=\|\beta_0^\star-\beta_1^\star\|_2$. Along fixed unit directions $d_g$, define $u_g(R)=\beta_g^\star+R\Delta_\beta d_g$ and center the initial cluster models near $\bar u(R)=\frac12\bigl(u_0(R)+u_1(R)\bigr)$.

\begin{wrapfigure}{r}{0.38\linewidth}
    \vspace{-0.8em}
    \centering
    \includegraphics[width=0.98\linewidth]{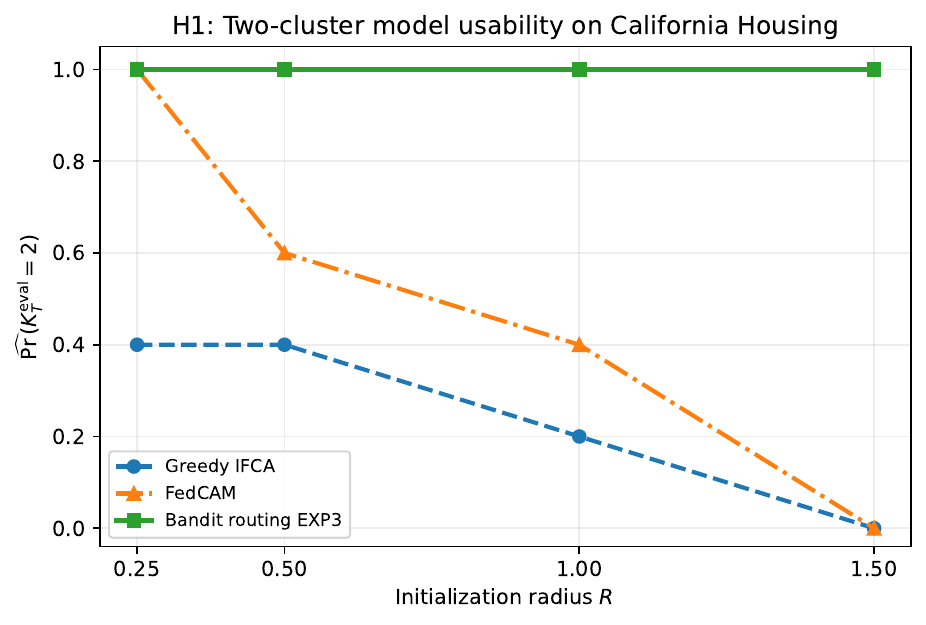}
    \caption{Housing model usability.}
    \label{fig:app-housing}
    \vspace{-0.8em}
\end{wrapfigure}

Figure~\ref{fig:app-housing} reports $\widehat{\Pr}(K_T^{\rm eval}=2)$, the fraction of runs in which both frozen terminal models receive at least one controller free greedy client. EXP3 retains two usable models throughout the tested radius sweep, while Greedy IFCA and FedCAM increasingly collapse as the initialization radius $R$ grows.

Housing uses $\Pi_{\rm house}=\{\phi_{\rm gr},\phi_{\rm rev},\phi_{\rm com}\}$. For low mass revival, fix $\mu_{\rm rev}\in(0,1)$ and define $\mathcal K_{\rm weak}^t=\{k\in[K]:\mu_k^{t-1}<\mu_{\rm rev}\}$. Let $k_{\min}(t)\in\arg\min_{k\in[K]}\mu_k^{t-1}$ be a least used cluster. The revival controller is
\[
\phi_{\rm rev}(H_i^t)
=
\begin{cases}
e_{k_{\min}(t)}, & \mathcal K_{\rm weak}^t\neq\varnothing,\\
e_{k_i^{(1)}(t)}, & \text{otherwise}.
\end{cases}
\]
Thus revival routes the client to the least used cluster whenever some cluster has mass below $\mu_{\rm rev}$; otherwise it uses Greedy.

A single nongreedy update may not make a weak model competitive, so the commitment controller lets client $i$ retain target $c_i^t$ for several rounds. Let $\tau_i^t$ be the remaining duration and $\kappa_{\rm com}\ge 1$ the accepted loss tolerance. The controller is
\[
\phi_{\rm com}(H_i^t)
=
\begin{cases}
e_{c_i^t},
&
\tau_i^t>0
\ \text{and}\
\widehat L_i(w_{c_i^t}^t)
\le
\kappa_{\rm com}
\widehat L_i\!\left(w_{k_i^{(1)}(t)}^t\right),\\
e_{k_i^{(1)}(t)},
&
\text{otherwise}.
\end{cases}
\]
For a committed nongreedy route, the client stores its target and remaining duration. Several updates can then accumulate before reevaluation, while the loss tolerance prevents continued commitment to an excessively poor target.
\FloatBarrier

\end{document}